\documentclass[11pt, notitlepage]{article}      

\usepackage{amsfonts,amsmath,amssymb,amscd,amsthm,euscript,graphicx,latexsym,mathrsfs,verbatim,setspace} 

\usepackage{authblk}

\theoremstyle{plain}
\newtheorem{thm}{Theorem}[section]
\newtheorem{lem}[thm]{Lemma}
\newtheorem{cor}[thm]{Corollary}
\newtheorem{conj}[thm]{Conjecture}

\newtheorem{prop}[thm]{Proposition}
\usepackage{mathtools}
\usepackage{amsthm}
\usepackage{amssymb}
\usepackage{amsfonts}
\usepackage{graphicx}
\usepackage{ytableau}
\usepackage{hyperref}
\usepackage{cleveref}
\usepackage[utf8]{inputenc}
\usepackage{todonotes}
\usepackage{tikz}
\usepackage{thmtools}
\usepackage{thm-restate}
\usepackage{comment}

\usepackage{tikz}
\usepackage{pgfplots}
\pgfplotsset{compat=1.5}

\newcommand{\opt}{\operatorname{opt}}

\begin{document}

\title{Online Learning of Smooth Functions} 

\author{Jesse Geneson and Ethan Zhou}

\maketitle

\begin{abstract}
In this paper, we study the online learning of real-valued functions where the hidden function is known to have certain smoothness properties. Specifically, for $q \ge 1$, let $\mathcal F_q$ be the class of absolutely continuous functions $f: [0,1] \to \mathbb R$ such that $\|f'\|_q \le 1$. For $q \ge 1$ and $d \in \mathbb Z^+$, let $\mathcal F_{q,d}$ be the class of functions $f: [0,1]^d \to \mathbb R$ such that any function $g: [0,1] \to \mathbb R$ formed by fixing all but one parameter of $f$ is in $\mathcal F_q$. For any class of real-valued functions $\mathcal F$ and $p>0$, let $\opt_p(\mathcal F)$ be the best upper bound on the sum of $p^{\text{th}}$ powers of absolute prediction errors that a learner can guarantee in the worst case. In the single-variable setup, we find new bounds for $\opt_p(\mathcal F_q)$ that are sharp up to a constant factor. We show for all $\varepsilon \in (0, 1)$ that $\opt_{1+\varepsilon}(\mathcal{F}_{\infty}) = \Theta(\varepsilon^{-\frac{1}{2}})$ and $\opt_{1+\varepsilon}(\mathcal{F}_q) = \Theta(\varepsilon^{-\frac{1}{2}})$ for all $q \ge 2$. We also show for $\varepsilon \in (0,1)$ that $\opt_2(\mathcal F_{1+\varepsilon})=\Theta(\varepsilon^{-1})$. In addition, we obtain new exact results by proving that $\opt_p(\mathcal F_q)=1$ for $q \in (1,2)$ and $p \ge 2+\frac{1}{q-1}$. In the multi-variable setup, we establish inequalities relating $\opt_p(\mathcal F_{q,d})$ to $\opt_p(\mathcal F_q)$ and show that $\opt_p(\mathcal F_{\infty,d})$ is infinite when $p<d$ and finite when $p>d$. We also obtain sharp bounds on learning $\mathcal F_{\infty,d}$ for $p < d$ when the number of trials is bounded.
\end{abstract}

\section{Introduction}

Consider a learner that wants to predict the next day’s temperature range at a given location based on inputs such as the current day's temperature range, humidity, atmospheric pressure, precipitation, wind speed, solar radiation, location, and time of year. In our model, this learner is tested daily. On a given day, the learner gets inputs for that day, which it uses to output a prediction for the next day’s temperature range; when the next day arrives, it sees the correct temperature range, then uses this feedback to update future predictions. As this is repeated, the learner accumulates information to help it make better predictions. A natural question arises: can the learner guarantee that its predictions become better over time, and if so, how quickly?

We investigate a model of online learning of real-valued functions previously studied in \cite{kl,long,mycielski,angluin,littlestone,lw} where an algorithm $A$ learns a real-valued function $f$ from some class $\mathcal F$ in trials. Past research on this model focused on functions of one input, for example, predicting the temperature range solely based on the time of year. The research showed that, as long as the function is sufficiently smooth, the learner can become a good predictor fairly rapidly. Suppose that $\mathcal F$ consists of functions $f: S \to \mathbb R$ for some set $S$, and fix some $f \in \mathcal F$. In each trial $t=0,\ldots,m$, $A$ receives an input $s_t \in S$, guesses $\hat y_t$ for the value of $f(s_t)$, and receives the actual value of $f(s_t)$. 

Following \cite{kl}, we focus on an error function which measures how difficult it is for a learner to predict functions accurately in the worst case. The error function depends on two parameters, $p$ and $q$, which determine how harshly the learner is punished for errors and the types of functions that the learner might encounter, respectively. Small values of $p$ and $q$ are more difficult for the learner, leading to higher values of the error function. For each algorithm $A$, $p > 0$, $f \in \mathcal F$, and $\sigma=(s_0,\ldots,s_m) \in S^{m+1}$, define \[ \mathscr L_p(A,f,\sigma)=\sum_{t=1}^m|\hat y_t-f(s_t)|^p. \] When $f$ and $\sigma$ are clear from the context, we refer to $\mathscr L_p(A,f,\sigma)$ as \textit{the total $p$-error of $A$}. Define \[ \mathscr L_p(A,\mathcal F)=\displaystyle\sup_{f \in \mathcal F,\sigma \in \cup_{m \in \mathbb Z^+}S^m}\mathscr L_p(A,f,\sigma) \] and $\opt_p(\mathcal F)=\displaystyle\inf_A \mathscr L_p(A,\mathcal F)$. Note that unlike the definition of $\opt$ presented in \cite{kl,long,geneson}, $\mathcal F$ may consist of real-valued functions on any domain, not just functions from $[0,1]$ to $\mathbb R$.

The case where $\mathcal F$ contains functions $f: [0,1] \to \mathbb R$ whose derivatives have various bounded norms was studied in \cite{kl,long,geneson}. For $q \ge 1$, let $\mathcal F_q$ be the class of absolutely continuous functions $f: [0,1] \to \mathbb R$ such that $\int_0^1|f'(x)|^q \text{d}x \le 1$, and let $\mathcal F_\infty$ be the class of absolutely continuous functions $f: [0,1] \to \mathbb R$ such that $\displaystyle \sup_{x \in (0,1)}|f'(x)| \le 1$. As noted in \cite{long}, $\mathcal F_\infty$ contains exactly those $f: [0,1] \to \mathbb R$ such that $|f(x)-f(y)| \le |x-y|$ for all $x,y \in [0,1]$. Also, by Jensen's inequality, $\mathcal F_\infty \subseteq \mathcal F_q \subseteq \mathcal F_r$ for all $q \ge r \ge 1$. Hence $\opt_p(\mathcal F_\infty) \le \opt_p(\mathcal F_q) \le \opt_p(\mathcal F_r)$ for all $p \ge 1$ and $q \ge r \ge 1$. Previous papers determined the exact values of $\opt_p(\mathcal F_q)$ for $p = 1$, $q = 1$, and $p, q \ge 2$, as well as bounds on $\opt_p(\mathcal F_q)$ for $p \in (1, 2)$ and $q \ge 2$.

\begin{figure}[t]
        \centering
        \begin{tikzpicture}
\begin{axis}[
    tick label style={font=\scriptsize},
    axis y line*=left,
    axis x line*=bottom,
    xmin=1,
    xmax=5,
    ymin=1,
    ymax=5,
    xtick={1,2},
    xtick style={draw=none},
    ytick style={draw=none},
    ytick={1,2},
	xlabel=$p$,
	ylabel=$q$
]
\addplot [black] coordinates {
	(2,2)
	(2,5)
};
\addplot [black] coordinates {
	(1,2)
	(5,2)
};
\addplot [black] coordinates {
	(1,1)
	(1,5)
};
\addplot [black] coordinates {
	(1,1)
	(5,1)
};
\node at (axis cs:3.5,3.5) { $1$ };
\node[scale=0.8] at (axis cs:1.5,3.5) { $O\left(\frac{1}{p-1}\right)$ };
\end{axis}
\end{tikzpicture}
        \caption{Exact values and bounds on $\opt_p(\mathcal{F}_q)$ for $p, q > 1$ prior to the results in this paper}
        \label{old_bounds}
    \end{figure}
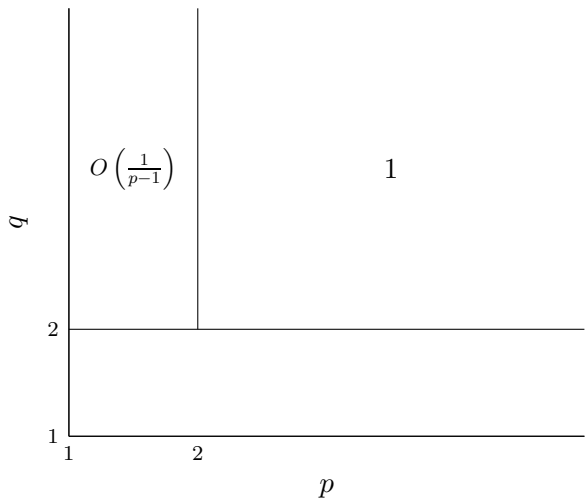

The paper \cite{kl} proved that $\opt_p(\mathcal{F}_1) = \infty$ for all $p \ge 1$. They also showed that $\opt_1(\mathcal{F}_q) = \opt_1(\mathcal{F}_{\infty}) = \infty$ for all $q \ge 1$. In contrast, they found that $\opt_p(\mathcal{F}_q) = \opt_p(\mathcal{F}_{\infty}) = 1$ for all $p \ge 2$ and $q \ge 2$. This was also proved in \cite{FM} using a different algorithm based on a generalization of the Widrow-Hoff algorithm \cite{kaczmarz, wh}, and a noisy version of this problem was studied in \cite{clw}. In this paper, we extend the region of values of $p, q$ for which it is known that $\opt_p(\mathcal{F}_q) = 1$.

\begin{thm}\label{more1bound}
For any reals $q>1$ and $p \ge 2+\frac{1}{q-1}$, we have $\opt_p(\mathcal F_q)=1$.
\end{thm}

For $p = 1+\varepsilon$ with $\varepsilon \in (0, 1)$, the paper \cite{kl} proved that $\opt_p(\mathcal{F}_q) = O(\varepsilon^{-1})$ for all $q \ge 2$, which implies that $\opt_p(\mathcal{F}_{\infty}) = O(\varepsilon^{-1})$. However, these bounds are not sharp. In this paper, we determine $\opt_{1+\varepsilon}(\mathcal{F}_q)$ up to a constant factor for all $\varepsilon \in (0, 1)$ and $q \ge 2$. 

\begin{thm}\label{mainth}
For all $\varepsilon \in (0, 1)$, we have $\opt_{1+\varepsilon}(\mathcal{F}_{\infty}) = \Theta(\varepsilon^{-\frac{1}{2}})$ and $\opt_{1+\varepsilon}(\mathcal{F}_q) = \Theta(\varepsilon^{-\frac{1}{2}})$ for all $q \ge 2$, where the constants in the bound do not depend on $q$.
\end{thm}

The proof of Theorem \ref{mainth} splits into an upper bound and a lower bound. For the upper bound, we use H\"older's inequality combined with results from \cite{kl}. For the lower bound, we modify a construction used in \cite{long}, which obtained bounds on a finite variant of $\opt_1(\mathcal{F}_q)$ that depends on the number of trials $m$. 

The results of \cite{kl} and \cite{long} left open the problem of determining $\opt_p(F_q)$ for $q \in (1, 2)$. It was not even known up to a constant factor. We make progress on this problem by determining $\opt_2(\mathcal F_{1+\varepsilon})$ up to a constant factor for $\varepsilon \in (0, 1)$. Figure \ref{old_bounds} shows the bounds and exact values known for $p, q > 1$ prior to the results in our paper, while Figure \ref{new_bound} shows the bounds and exact values known for $p, q > 1$ including the results in our paper.

\begin{thm}\label{qin12bound}
For $\varepsilon \in (0,1)$, we have $\opt_2(\mathcal F_{1+\varepsilon})=\Theta(\varepsilon^{-1})$.
\end{thm}

The paper \cite{kl} also discussed the problem of online learning for smooth functions of multiple variables. Previous research on learning multi-variable functions \cite{barron, hardle, haussler} has focused on expected loss rather than worst-case loss, using models where the inputs $x_i$ are determined by a probability distribution. 

\begin{figure}[t]
        \centering
        \begin{tikzpicture}
\begin{axis}[
    tick label style={font=\scriptsize},
    axis y line*=left,
    axis x line*=bottom,
    xmin=1,
    xmax=5,
    ymin=1,
    ymax=5,
    xtick={1,2},
    xtick style={draw=none},
    ytick style={draw=none},
    ytick={1,2},
	xlabel=$p$,
	ylabel=$q$
]
\addplot [black] coordinates {
	(2,1)
	(2,5)
};
\addplot [black] coordinates {
	(1,2)
	(5,2)
};
\addplot [black] coordinates {
	(1,1)
	(1,5)
};
\addplot [black] coordinates {
	(1,1)
	(5,1)
};
\addplot [black,domain=3:5] {
    1+1/(x-2)
};
\node at (axis cs:3.5,3.5) { $1$ };
\node[scale=0.8] at (axis cs:1.5,3.5) { $\Theta\left(\frac{1}{\sqrt{p-1}}\right)$ };
\node at (axis cs:4.5,1.7) { $1$ };
\node[scale=0.8] at (axis cs:2.7,1.5) { $O\left(\frac{1}{q-1}\right)$ };
\end{axis}
\end{tikzpicture}
        \caption{Exact values and bounds on $\opt_p(\mathcal{F}_q)$ for $p, q > 1$ including the results in this paper}
        \label{new_bound}
    \end{figure}
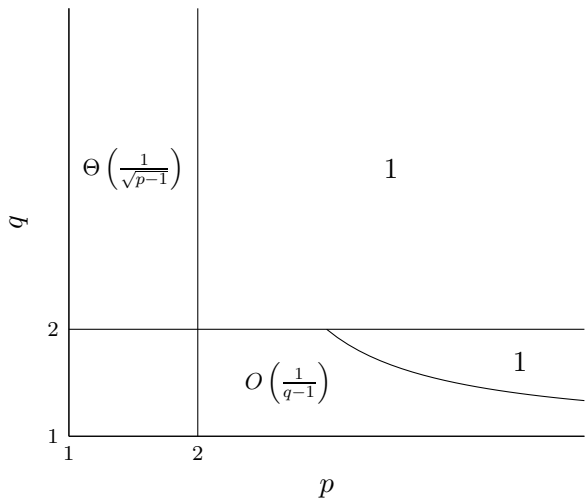

We introduce a natural extension of the single-variable setup from \cite{kl} to multi-variable functions. Specifically, for $q \ge 1$ and $d \in \mathbb Z^+$, let $\mathcal F_{q,d}$ be the class of functions $f: [0,1]^d \to \mathbb R$ such that for any $(d-1)$-tuple $(x_1,\ldots,x_{d-1}) \in [0,1]^{d-1}$ and integer $i$ with $1 \le i \le d$, the function $g: [0,1] \to \mathbb R$ given by $g(x)=f(\mathbf v_{i,x})$ is in $\mathcal F_q$, where $\mathbf v_{i,x} \in [0,1]^d$ is the vector formed when $x$ is inserted at the $i^{\text{th}}$ position of $(x_1,\ldots,x_{d-1})$.

One of the most fundamental questions about $\opt_p(\mathcal F_{q,d})$ is to determine when it is finite and when it is infinite. We answer this question almost completely when $q = \infty$.

\begin{thm}\label{mvinfty}
For any positive integer $d$, $\opt_p(\mathcal F_{\infty,d})$ is finite when $p > d$ and infinite when $0 < p < d$.
\end{thm}

As a corollary, it immediately follows for $0 < p < d$ that $\opt_p(\mathcal F_{q,d}) = \infty$ for all $q \ge 1$. Moreover, it is easy to see that $\opt_p(\mathcal F_{1,d}) = \infty$ for all positive integers $d$ and $p$.

The papers \cite{kl} and \cite{long} also investigated worst-case mistake bounds for online learning of smooth functions when the number of trials is bounded. In particular, using the same notation as in the first paragraph of this section, define \[ \mathscr L_p(A,f,\sigma, m)=\sum_{t=1}^m|\hat y_t-f(s_t)|^p. \] Moreover, define \[ \mathscr L_p(A,\mathcal F,m)=\displaystyle\sup_{f \in \mathcal F,\sigma \in S^{m+1}}\mathscr L_p(A,f,\sigma,m) \] and $\opt_p(\mathcal F,m)=\displaystyle\inf_A \mathscr L_p(A,\mathcal F,m)$. 

The paper \cite{kl} proved that $\opt_1(\mathcal F_q,m) = O(\log(m))$ for all $q \ge 2$ and $\opt_1(\mathcal F_2,m) = \Omega(\sqrt{\log(m)})$. The paper \cite{long} sharpened these bounds by proving that $\opt_1(\mathcal F_q,m) = \Theta(\sqrt{\log(m)})$ for all $q \ge 2$ and $\opt_1(\mathcal F_2,m) = \frac{\sqrt{\log_2(m)}}{2} \pm O(1)$. We obtain sharp bounds for online learning of smooth functions with a bounded number of trials when $0 < p < d$. In particular, these sharp bounds are also new in the single-variable case.

\begin{thm}\label{bounded_m_trials}
For any positive integer $d$ and real number $p$ with $0 < p<d$, we have $\opt_p(\mathcal F_{\infty,d}, m)=\Theta(m^{1-\frac{p}{d}})$, where the constants in the bounds depend on $p$ and $d$.
\end{thm}

In Section \ref{2}, we focus on the single-variable setup. We prove Theorem \ref{qin12bound} in Subsections \ref{pqlower} and \ref{2qupper}. Subsection \ref{pqlower} establishes the lower bound, while Subsection \ref{2qupper} establishes the upper bound along with several useful lemmas. Subsection \ref{pqupper} focuses on proving Theorem \ref{more1bound}. In Subsection \ref{pql2}, we prove Theorem \ref{mainth}. In Section \ref{3}, we focus on the multi-variable setup, establishing various bounds on $\opt_p(\mathcal F_{q,d})$. Finally, in Section \ref{s:open}, we discuss open problems.

\section{Results in the single-variable setup for $q \in (1,2)$}\label{2}

First, we adopt some notation from \cite{kl}. For $f: [0,1] \to \mathbb R$, define the \textit{$q$-action} of $f$, denoted by $J_q[f]$, as \[ J_q[f]=\int_0^1 |f'(x)|^q\text{d}x, \] so that $\mathcal F_q$ is exactly the set of absolutely continuous $f: [0,1] \to \mathbb R$ such that $J_q[f] \le 1$.

Also, for a nonempty set $S=\{(u_i,v_i): 1 \le i \le m\}$ of points in $[0,1] \times \mathbb R$ such that $u_1<\ldots<u_m$, define \[ f_S(x)=\begin{cases} v_1 & x \le u_1 \\ v_i+\frac{(x-u_i)(v_{i+1}-v_i)}{u_{i+1}-u_i} & x \in (u_i,u_{i+1}] \\ v_m & x>u_m \end{cases} \] and set $f_{\emptyset}(x) \equiv 0$.

Finally, define the learning algorithm LININT as follows: on trial $0$, LININT guesses $\hat y_0=0$, and on trial $i>0$, with the points in $S=\{(x_0,f(x_0)),\ldots,(x_{i-1},f(x_{i-1}))\}$ having been revealed and given $x_i$, LININT guesses $\hat y_i=f_S(x_i)$.

\subsection{Lower bounds for $\opt_p(\mathcal F_q)$}\label{pqlower}

First, for all $p,q>1$ we have an obvious lower bound for $\opt_p(\mathcal F_q)$.

\begin{prop}\label{pq1}
For $p,q>1$, we have $\opt_p(\mathcal F_q) \ge 1$.
\end{prop}

\iffalse
\begin{proof}
Consider the adversary strategy where the adversary picks $x_0=0$ and reveals $f(x_0)=0$, then picks $x_1=1$ and reveals $f(x_1)=\pm 1$ such that the error $|\hat y_1-f(x_1)|$ is at least 1. This is consistent with one of the functions $\{f(x)=x,f(x)=-x\} \subset \mathcal F_q$ and guarantees $\sum_{i \ge 1} |\hat y_i-f(x_i)|^p \ge 1$.
\end{proof}
\else

The paper \cite{kl} proved that equality holds when $p,q \ge 2$. As we will see, equality also holds when $q \in (1,2)$ for sufficiently large values of $p$.

For $q \in (1,2)$ and $p>1$, we also prove a lower bound for $\opt_2(\mathcal F_q)$, using an adversary strategy similar to that in Theorem 8 of \cite{kl}.

\begin{thm}\label{pqlow}
For $q \in (1,2)$, we have $\opt_p(\mathcal F_q) \ge \frac{q}{(p2^pe \ln 2)(q-1)}$.
\end{thm}

\begin{proof}
Fix $q \in (1,2)$ and an algorithm $A$ for learning $\mathcal F_q$. Consider the following family of adversary strategies, depending on a parameter $b \in (0,1)$. The adversary picks $x_0=0$ and reveals $f(x_0)=0$, then picks $x_1=1$ and reveals $f(x_1)=\pm b$ such that $|\hat y_1-f(x_1)| \ge b$; without loss of generality, suppose $f(x_1)=b$. Then for the next $k=\left\lfloor -\frac{q\log_2 b}{q-1} \right\rfloor$ trials, the adversary recursively picks $x_i$ and $f(x_i)$ as follows. On trial $2 \le i<k+2$, the adversary sets $l_i$ to be the greatest real $x \in [0,1]$ such that $f(x)=0$ has been previously revealed, and similarly sets $r_i$ to be the least real $x \in [0,1]$ such that $f(x)=b$ has been previously revealed, then sets $x_i=\frac{l_i+r_i}{2}$. Upon receiving $A$'s guess $\hat y_i$, the adversary reveals $f(x_i)=0$ or $f(x_i)=b$ such that $|f(x_i)-\hat y_i| \ge \frac{b}{2}$.

To see that this strategy is well-defined, note that all the $x_i$ are distinct\iffalse (for example, there exists $r \in [0,1]$ such that the sequence $x_2,\ldots,x_{k+1}$ is exactly the sequence of guesses made by a binary search algorithm trying to guess $r$)\else, so it suffices to show that there exists a function $f \in \mathcal F_q$ which is consistent with all $(x_i,f(x_i))$. Indeed, take $f=f_{\{(x_0,f(x_0)),\ldots,(x_{k+1},f(x_{k+1}))\}}$ which linearly interpolates between all points $(x_i,f(x_i))$; then $f$ has only one segment of nonzero slope, with \[ J_q[f]=2^{-k}\left(\frac{b}{2^{-k}}\right)^q=2^{k(q-1)}b^q \le 2^{-\frac{q\log_2 b}{q-1} \cdot (q-1)}b^q=1. \] Thus the adversary guarantees an error of at least \[ \opt_p(\mathcal F_q) \ge \sum_{i=1}^{k+1}|f(x_i)-\hat y_i|^p \ge b^p+k\left(\frac{b}{2}\right)^p \ge -\frac{b^pq\log_2 b}{2^p(q-1)}. \] Picking $b=e^{-\frac{1}{p}}$ yields $\opt_p(\mathcal F_q) \ge \frac{q}{(p2^pe \ln 2)(q-1)}$.
\end{proof}

In particular, when $p=2$ we get the following:

\begin{cor}\label{2qlow}
For $q \in (1,2)$, we have \[\opt_2(\mathcal F_q) \ge \frac{q}{(8e\ln 2)(q-1)}>\frac{1}{(8e\ln 2)(q-1)}.\]
\end{cor}

\subsection{Bounds for $\opt_2(\mathcal F_q)$}\label{2qupper}

The main result of this section is that for $\varepsilon \in (0,1)$, $\opt_2(\mathcal F_{1+\varepsilon})=\Theta(\varepsilon^{-1})$. For $q \in (1,2)$, Corollary \ref{2qlow} gives a lower bound for $\opt_2(\mathcal F_q)$; we now prove an upper bound for $\opt_2(\mathcal F_q)$ and use this to derive the desired result. First, we show that a similar fact to Lemma 9 in \cite{kl} holds.

\begin{lem}\label{fsmin}
Let $u_1<\ldots<u_m$ be reals in $[0,1]$ and $v_1,\ldots,v_m$ be reals, and define $S=\{(u_1,v_1),\ldots,(u_m,v_m)\}$. Then for any $q \in (1,2)$ and absolutely continuous $f: [0,1] \to \mathbb R$ such that $f(u_i)=v_i$ for $1 \le i \le m$, we have $J_q[f] \ge J_q[f_S]$.
\end{lem}

\begin{proof}
If $m=1$, then $J_q[f_S]=0$ and the result is clear. Otherwise, fix some absolutely continuous $f: [0,1] \to \mathbb R$ which is consistent with the $(u_i,v_i)$, and fix $1 \le i<m$. Then by Jensen's inequality, \[ \frac{\int_{u_i}^{u_{i+1}}|f'(x)|^q\text{d}x}{u_{i+1}-u_i} \ge \left(\frac{\int_{u_i}^{u_{i+1}}|f'(x)|\text{d}x}{u_{i+1}-u_i}\right)^q \ge \left(\frac{\left|\int_{u_i}^{u_{i+1}}f'(x)\text{d}x\right|}{u_{i+1}-u_i}\right)^q=\left|\frac{v_{i+1}-v_i}{u_{i+1}-u_i}\right|^q. \] Thus we obtain \[ J_q[f] \ge \int_{u_1}^{u_m}|f'(x)|^q\text{d}x \ge \sum_{i=1}^{m-1}(u_{i+1}-u_i)\left|\frac{v_{i+1}-v_i}{u_{i+1}-u_i}\right|^q=J_q[f_S] \] by summing over all $1 \le i<m$.
\end{proof}

This leads to the following useful fact.

\begin{lem}\label{linint1}
For any $q>1$, target function $f \in \mathcal F_q$, integer $m \ge 1$, and sequence of inputs $x_0,\ldots,x_m \in [0,1]$, \textnormal{LININT} never produces an error $|\hat y_i-f(x_i)|>1$ on any trial $i \ge 1$.
\end{lem}

\begin{proof}
Suppose otherwise, so that LININT produces an error $|\hat y_i-f(x_i)|>1$ for $i \ge 1$; then there exists $0 \le j<i$ such that $|f(x_i)-f(x_j)|>1$. Letting $S=\{(x_0,f(x_0)),\ldots,(x_i,f(x_i))\}$, by Lemma \ref{fsmin} \[ J_q[f] \ge J_q[f_S] \ge |x_i-x_j|\left|\frac{f(x_i)-f(x_j)}{x_i-x_j}\right|^q \ge |f(x_i)-f(x_j)|>1, \] upon which $f \not \in \mathcal F_q$, contradiction.
\end{proof}

\begin{cor}\label{linintp}
For any $q>1$ and $p'>p>1$, we have $\mathscr L_{p'}(\textnormal{LININT},\mathcal F_q) \le \mathscr L_p(\textnormal{LININT},\mathcal F_q)$.
\end{cor}

\begin{proof}
On every trial $i \ge 1$, \textnormal{LININT} produces an error $|\hat y_i-f(x_i)| \le 1$, so $|\hat y_i-f(x_i)|^{p'} \le |\hat y_i-f(x_i)|^{p}$ for all $i \ge 1$.
\end{proof}

With this, the proof proceeds similarly to the proof of Theorem 11 in \cite{kl}. Specifically, we will compare changes in $J_q[f_S]$ as new points are added to $S$ to the squared errors $(\hat y-f_S(x))^2$ produced by LININT to bound $\mathscr L_2(\text{LININT},\mathcal F_q)$. This requires the following inequalities.

\begin{lem}\label{2qin}
For reals $a > 0$, $b>0$, $q \in (1,2)$, and $x \in (-a,b)$, we have \[ a\left(1+\frac{x}{a}\right)^q+b\left(1-\frac{x}{b}\right)^q-(a+b) \ge \frac{2q(q-1)}{a+b} \cdot x^2. \]
\end{lem}

\begin{proof}
Fix $a,b,q$, and define the function \[ f(x)=a\left(1+\frac{x}{a}\right)^q+b\left(1-\frac{x}{b}\right)^q-(a+b)-\frac{2q(q-1)}{a+b} \cdot x^2 \] for $x \in (-a,b)$, so that we wish to show $f(x) \ge 0$ for all $x \in (-a,b)$. Compute \begin{align*} f'(x) &= q\left[\left(1+\frac{x}{a}\right)^{q-1}-\left(1-\frac{x}{b}\right)^{q-1}-\frac{4(q-1)}{a+b}x\right] \\ f''(x) &= q(q-1)\left[\frac{1}{a}\left(1+\frac{x}{a}\right)^{q-2}+\frac{1}{b}\left(1-\frac{x}{b}\right)^{q-2}-\frac{4}{a+b}\right] \\ f^{(3)}(x) &= q(q-1)(q-2)\left[\frac{1}{a^2}\left(1+\frac{x}{a}\right)^{q-3}-\frac{1}{b^2}\left(1-\frac{x}{b}\right)^{q-3}\right] \\ f^{(4)}(x) &= q(q-1)(q-2)(q-3)\left[\frac{1}{a^3}\left(1+\frac{x}{a}\right)^{q-4}+\frac{1}{b^3}\left(1-\frac{x}{b}\right)^{q-4}\right]. \end{align*} First, we show $f''(x) \ge 0$ for all $x \in (-a,b)$. Note that $f^{(4)}(x)>0$ for all $x$ and \[ \lim_{x \to -a^+}f''(x)=\lim_{x \to b-}f''(x)=\infty, \] so $f^{(3)}(x)$ is increasing on $(-a,b)$ and it suffices to check that $f''(x) \ge 0$ at the point where $f^{(3)}(x)=0$. Solving for this $x$ yields \begin{align*} f^{(3)}(x)=0 & \iff \frac{1}{a^2}\left(1+\frac{x}{a}\right)^{q-3}=\frac{1}{b^2}\left(1-\frac{x}{b}\right)^{q-3} \\ & \iff \left(\frac{1+\frac{x}{a}}{1-\frac{x}{b}}\right)^{3-q}=\frac{b^2}{a^2} \iff \frac{1+\frac{x}{a}}{1-\frac{x}{b}}=\frac{b^\frac{2}{3-q}}{a^\frac{2}{3-q}} \\ & \iff x=\frac{-a^\frac{2}{3-q}+b^\frac{2}{3-q}}{a^\frac{q-1}{3-q}+b^\frac{q-1}{3-q}}. \end{align*} At this $x$, \begin{align*} f''(x) &= q(q-1)\left[\frac{1}{a}\left(1+\frac{-a^\frac{2}{3-q}+b^\frac{2}{3-q}}{a^\frac{2}{3-q}+ab^\frac{q-1}{3-q}}\right)^{q-2}+\frac{1}{b}\left(1-\frac{-a^\frac{2}{3-q}+b^\frac{2}{3-q}}{a^\frac{q-1}{3-q}b+b^\frac{2}{3-q}}\right)^{q-2}-\frac{4}{a+b}\right] \\ &= q(q-1)\left[\frac{1}{a}\left(\frac{(a+b)b^\frac{q-1}{3-q}}{a\left(a^\frac{q-1}{3-q}+b^\frac{q-1}{3-q}\right)}\right)^{q-2}+\frac{1}{b}\left(\frac{(a+b)a^\frac{q-1}{3-q}}{b\left(a^\frac{q-1}{3-q}+b^\frac{q-1}{3-q}\right)}\right)^{q-2}-\frac{4}{a+b}\right] \\ &= q(q-1)\left[\frac{\left(a^\frac{q-1}{3-q}+b^\frac{q-1}{3-q}\right)^{2-q}\left(a^{1-q}b^\frac{(q-1)(q-2)}{3-q}+a^\frac{(q-1)(q-2)}{3-q}b^{1-q}\right)}{(a+b)^{2-q}}-\frac{4}{a+b}\right] \\ & \ge q(q-1)\left[\frac{2^{2-q}(ab)^\frac{(q-1)(2-q)}{2(3-q)} \cdot 2(ab)^\frac{(q-1)(2q-5)}{2(3-q)}}{(a+b)^{2-q}}-\frac{4}{a+b}\right] \\ &= 4q(q-1)\left[\frac{1}{(a+b)^{2-q}(4ab)^\frac{q-1}{2}}-\frac{1}{a+b}\right] \\ & \ge 4q(q-1)\left[\frac{1}{(a+b)^{2-q}(a+b)^{q-1}}-\frac{1}{a+b}\right]=0, \end{align*} where all inequalities follow from the inequality $(u+v)^2 \ge 4uv \iff (u-v)^2 \ge 0$ for all reals $u,v$. Since $f''(x)$ is minimized here, it follows that $f''(x) \ge 0$ for all $x \in (-a,b)$.

Since $f'(0)=0$ and $f''(x) \ge 0$ for all $x \in (-a,b)$, it follows that $f'(x) \le 0$ for $x<0$ and $f'(x) \ge 0$ for $x>0$, so $f(x) \ge f(0)=0$ for all $x \in (-a,b)$.
\end{proof}

\begin{lem}\label{2qout}
For reals $a,b \in (0,1)$, $q \in (1,2)$, and $x \not \in (-a,b)$, we have \[ a\left|\frac{x}{a}+1\right|^q+b\left|\frac{x}{b}-1\right|^q-(a+b) \ge \frac{(q-1)|x|^q}{(a+b)^{q-1}}. \]
\end{lem}

\begin{proof}
Fix $a,b,q$; by symmetry, it suffices to consider $x \ge b$. Define the function \[ f(x)=a\left(\frac{x}{a}+1\right)^q+b\left(\frac{x}{b}-1\right)^q-(a+b)-\frac{(q-1)x^q}{(a+b)^{q-1}} \] for $x \ge b$, so that we wish to show $f(x) \ge 0$ for all $x \ge b$. Since \begin{align*} f'(x) &= q\left(\frac{x}{a}+1\right)^{q-1}+q\left(\frac{x}{b}-1\right)^{q-1}-\frac{q(q-1)x^{q-1}}{(a+b)^{q-1}} \\ & \ge q\left(\frac{x}{a+b}\right)^{q-1}+q\left(\frac{x}{b}-1\right)^{q-1}-\frac{qx^{q-1}}{(a+b)^{q-1}}>0 \end{align*} for all $x>b$, $f$ is increasing, so it suffices to show \[ f(b)=a\left(\frac{a+b}{a}\right)^q-(a+b)-\frac{(q-1)b^q}{(a+b)^{q-1}} \ge 0. \] Dividing by $a+b$ and substituting $r=\frac{a}{a+b}$, we see that this is equivalent to \[ g(r)=\frac{1}{r^{q-1}}-1-(q-1)(1-r)^q \ge 0 \] for $r \in (0,1)$. As $g(1)>0$, it suffices for \begin{align*}
    & g'(r)=r^{-q}(q-1)\left(qr^q(1-r)^{q-1}-1\right)<0
    \\ & \iff qr^{q-1}(1-r)^{q-1}<\frac{1}{r}.
\end{align*} For $r \in (0, 1)$, the quantity $r^{q-1}(1-r)^{q-1}$ is maximized when $r = \frac{1}{2}$, so it suffices to prove that $q\left(\frac{1}{2}\right)^{2q-2} < 1$ for $q \in (1,2)$. Define $h(q)=q\left(\frac{1}{2}\right)^{2q-2}$, so $h'(q) = 4^{1-q} (1- q \ln{4}) < 0$ for $q \in (1, 2)$. Since $h(1) = 1$, we have $h(q) < 1$ for $q \in (1,2)$.
\end{proof}

\begin{cor}\label{2qoutcor}
For reals $a,b \in (0,1)$ such that $a+b \le 1$, $q \in (1,2)$, and $x \not \in (-a,b)$, we have \[ a\left|\frac{x}{a}+1\right|^q+b\left|\frac{x}{b}-1\right|^q-(a+b) \ge (q-1)|x|^q. \]
\end{cor}

Combining the above yields the following key result.

\begin{lem}\label{2qboth}
Fix $q \in (1,2)$, a nonempty set $S=\{(u_1,v_1),\ldots,(u_k,v_k)\}$ of points in $[0,1] \times \mathbb R$, and $(x,y) \in [0,1] \times \mathbb R$ such that $u_1<\ldots<u_k$, $x \neq u_i$ for any $1 \le i \le k$, and $J_q\left[f_{S \cup \{(x,y)\}}\right] \le 1$. Then \[ J_q\left[f_{S \cup \{(x,y)\}}\right]-J_q[f_S] \ge (q-1)(y-f_S(x))^2. \]
\end{lem}

\begin{proof}
First, suppose $x<u_1$. Then compared to $f_S$, the function $f_{S \cup \{(x,y)\}}$ contains a new line segment of (possibly) nonzero slope between $(x,y)$ and $(u_1,v_1)$, so as $|y-f_S(x)|=|v_1-y| \le 1$ (by Lemma \ref{linint1}) and $|u_1-x| \le 1$, \[ J_q\left[f_{S \cup \{(x,y)\}}\right]-J_q[f_S]=(u_1-x)\left|\frac{v_1-y}{u_1-x}\right|^q \ge (v_1-y)^2=(y-f_S(x))^2 \ge (q-1)(y-f_S(x))^2. \] The case $x>u_k$ is similar.

Now suppose there exists an integer $1 \le i<k$ such that $u_i<x<u_{i+1}$. In this case, \begin{align*} J_q\left[f_{S \cup \{(x,y)\}}\right]-J_q[f_S] &= (x-u_i)\left|\frac{y-v_i}{x-u_i}\right|^q+(u_{i+1}-x)\left|\frac{v_{i+1}-y}{u_{i+1}-x}\right|^q-(u_{i+1}-u_i)\left|\frac{v_{i+1}-v_i}{u_{i+1}-u_i}\right|^q. \end{align*} Substituting $a=x-u_i,b=u_{i+1}-x,d=y-f_S(x)$, and $m=\frac{v_{i+1}-v_i}{u_{i+1}-u_i}=\frac{f_S(x)-v_i}{a}=\frac{v_{i+1}-f_S(x)}{b}$, we can rewrite the above as \begin{align*} J_q\left[f_{S \cup \{(x,y)\}}\right]-J_q[f_S] &= a\left|m+\frac{d}{a}\right|^q+b\left|m-\frac{d}{b}\right|^q-(a+b)|m|^q \\ &= |m|^q\left(a\left|1+\frac{d}{ma}\right|^q+b\left|1-\frac{d}{mb}\right|^q-(a+b)\right). \end{align*} Then applying either Lemma \ref{2qin} or Corollary \ref{2qoutcor} (depending on whether $\frac{d}{m} \in (-a,b)$) yields \begin{align*} J_q\left[f_{S \cup \{(x,y)\}}\right]-J_q[f_S] & \ge |m|^q\min\left\{\frac{2q(q-1)}{a+b} \cdot \left(\frac{d}{m}\right)^2,(q-1)\left|\frac{d}{m}\right|^q\right\} \\ &= \min\left\{\frac{2q(q-1)}{|m|^{2-q}(a+b)} \cdot d^2,(q-1)|d|^q\right\}. \end{align*} If $0<|m| \le 1$, then $a+b=u_{i+1}-u_i \le 1 \implies |m|^{2-q}(a+b) \le 1$, while if $|m| \ge 1$, then \[ |m|^{2-q}(a+b) \le |m|^q(a+b) \le J_q\left[f_{S}\right] \le J_q[f_{S \cup \{(x,y)\}}] \le 1 \] by Lemma \ref{fsmin}, so in either case \[ \frac{2q(q-1)}{|m|^{2-q}(a+b)} \cdot d^2 \ge 2q(q-1)d^2 \ge (q-1)d^2. \] Moreover, since $J_q\left[f_{S \cup \{(x,y)\}}\right] \le 1$, by Lemma \ref{linint1} $|d| \le 1.$ Hence $(q-1)|d|^q \ge (q-1)d^2$ as well.
\end{proof}

This directly yields the desired upper bound.

\begin{thm}\label{2qup}
For $q \in (1,2)$, we have $\mathscr L_2(\normalfont\text{LININT},\mathcal F_q) \le \frac{1}{q-1}$.
\end{thm}

\begin{proof}
Fix a target function $f \in \mathcal F_q$, an integer $m \ge 1$, and a sequence of inputs $\sigma=(x_0,\ldots,x_m) \in [0,1]^{m+1}$. Assume without loss of generality that all $x_i$ are distinct. For $0 \le i \le m$, define $S_i=\{(x_0,f(x_0)),\ldots,(x_i,f(x_i))\}$, and suppose LININT produces guesses $\hat y_0,\ldots,\hat y_m \in \mathbb R$. By Lemma \ref{fsmin} and Lemma \ref{2qboth}, \[ 1 \ge J_q[f] \ge J_q\left[f_{S_m}\right]=\sum_{i=1}^m\left(J_q\left[f_{S_i}\right]-J_q\left[f_{S_{i-1}}\right]\right) \ge (q-1)\sum_{i=1}^m(\hat y_i-f(x_i))^2, \] so \[ \mathscr L_2(\text{LININT},f,\sigma)=\sum_{i=1}^m(\hat y_i-f(x_i))^2 \le \frac{1}{q-1} \] for any $f \in \mathcal F_q$, integer $m \ge 1$, and $\sigma \in [0,1]^{m+1}$. Thus $\mathscr L_2(\text{LININT},\mathcal F_q) \le \frac{1}{q-1}$.
\end{proof}

Finally, combining the above with the lower bound in Corollary \ref{2qlow}, we get the following result.

\begin{thm}\label{2q}
For $\varepsilon \in (0,1)$, we have $\opt_2(\mathcal F_{1+\varepsilon})=\Theta(\varepsilon^{-1})$.
\end{thm}

\begin{proof}
Combining Corollary \ref{2qlow} and Theorem \ref{2qup}, \[\frac{\varepsilon^{-1}}{8e \ln 2}<\frac{1+\varepsilon}{(8e \ln 2)\varepsilon} \le \opt_2(\mathcal F_{1+\varepsilon}) \le \varepsilon^{-1}. \] Hence, $\opt_2(\mathcal F_{1+\varepsilon})=\Theta(\varepsilon^{-1})$.\end{proof}

It is simple to generalize the upper bound in Theorem \ref{2qup} to all $p \ge 2$. 

\begin{cor}\label{finitepg2}
For $\varepsilon \in (0,1)$ and $p \ge 2$, we have $\opt_p(\mathcal F_{1+\varepsilon})=O(\varepsilon^{-1})$.
\end{cor}
\begin{proof}
By Lemma \ref{linintp}, $\opt_p(\mathcal F_{1+\varepsilon}) \le \mathscr L_p(\text{LININT},\mathcal F_{1+\varepsilon}) \le \mathscr L_2(\text{LININT},\mathcal F_{1+\varepsilon})=O(\varepsilon^{-1})$.
\end{proof}

\subsection{An exact result for large $p$}\label{pqupper}

In this section, we prove that for $q \in (1,2)$ and $p \ge 2+\frac{1}{q-1}$, $\opt_p(\mathcal F_q)=1$. This first requires the following lemma.

\begin{lem}\label{uv}
For reals $q \in (1,2)$, $a \in (0,1)$, and $u,v$ satisfying $|u-v| \ge \frac{(q-1)^{q-1}}{a(1-a)}$, we have \[ a|u|^q+(1-a)|v|^q>1. \]
\end{lem}

\begin{proof}
Without loss of generality, suppose $u>v$, so that $u \ge v+\frac{(q-1)^{q-1}}{a(1-a)}$.

First, suppose $v<0<u$; then $|u|+|v| \ge \frac{(q-1)^{q-1}}{a(1-a)}$. By the weighted power mean inequality, \begin{align*} \frac{|u|(a|u|^{q-1})+|v|((1-a)|v|^{q-1})}{|u|+|v|} & \ge \left(\frac{a^{-\frac{1}{q-1}}+(1-a)^{-\frac{1}{q-1}}} {|u|+|v|}\right)^{-(q-1)} \\ \implies a|u|^q+(1-a)|v|^q & \ge \frac{(|u|+|v|)^q}{\left(a^{-\frac{1}{q-1}}+(1-a)^{-\frac{1}{q-1}}\right)^{q-1}} \\ & \ge \frac{(q-1)^{q(q-1)}}{a^q(1-a)^q\left(a^{-\frac{1}{q-1}}+(1-a)^{-\frac{1}{q-1}}\right)^{q-1}} \\ & \ge \frac{(q-1)^{q(q-1)}}{2^{q-1}a^q(1-a)^q\max\left\{a^{-1},(1-a)^{-1}\right\}} \\ &= \frac{(q-1)^{q(q-1)}}{2^{q-1}\max\left\{a^{q-1}(1-a)^q,a^q(1-a)^{q-1}\right\}}. \end{align*} 
By the weighted arithmetic mean - geometric mean inequality, for $r \in (0, 1)$ we have \begin{align*} r^q(1-r)^{q-1} &= \frac{q^q}{(q-1)^q}\left(\frac{(q-1)r}{q}\right)^q(1-r)^{q-1} \\ & \le \frac{q^q}{(q-1)^q}\left(\frac{q \cdot \frac{(q-1)r}{q}+(q-1)(1-r)}{(q-1)+q}\right)^{(q-1)+q} \\ &= \frac{q^q}{(q-1)^q}\left(\frac{q-1}{2q-1}\right)^{2q-1}. \end{align*}
Thus \[ \max\left\{a^{q-1}(1-a)^q,a^q(1-a)^{q-1}\right\} \le \frac{q^q}{(q-1)^q}\left(\frac{q-1}{2q-1}\right)^{2q-1}=\frac{q^q(q-1)^{q-1}}{(2q-1)^{2q-1}}, \] so \[ a|u|^q+(1-a)|v|^q \ge \frac{(q-1)^{(q-1)^2}(2q-1)^{2q-1}}{2^{q-1}q^q}. \] Consider \[ f(q)=(q-1)^2\ln(q-1)+(2q-1)\ln(2q-1)-(q-1)\ln 2-q\ln q \] over $q \in (1,2)$. Note that \begin{align*} f'(q) &= (2(q-1)\ln(q-1)+(q-1))+(2\ln(2q-1)+2)-\ln 2-(\ln q+1) \\ &= 1-\ln 2+2(q-1)\ln(q-1)+(q-1-\ln q)+2\ln(2q-1) \\ & \ge 1-\ln 2+2(q-1)\ln(q-1)+2\ln(2q-1), \end{align*} as $e^x \ge 1+x$ implies that $x \ge \ln(1+x)$ for $x>-1$. Since $x\ln x$ is decreasing on $\left(0,\frac{1}{e}\right)$ and increasing on $\left(\frac{1}{e},\infty\right)$ (so in particular $x \ln x \ge -\frac{1}{e}$ for $x>0$), \begin{align*} q \in (1,1.004] & \implies f'(q) \ge 1-\ln 2+0.008\ln 0.004>0 \\ q \in [1.004,1.055] & \implies f'(q) \ge 1-\ln 2+0.11\ln 0.055+2\ln 1.008>0 \\ q \in [1.055,1.12] & \implies f'(q) \ge 1-\ln 2+0.24\ln 0.12+2\ln 1.11>0 \\ q \in [1.12,2) & \implies f'(q) \ge 1-\ln 2-\frac{2}{e}+2\ln 1.24>0, \end{align*} so for all $q \in (1,2)$, $f'(q)>0$. As $\displaystyle\lim_{q \to 1^+}f(q)=0$, it follows that $f(q)>0$ for $q \in (1,2)$, so $a|u|^q+(1-a)|v|^q \ge e^{f(q)}>1$ whenever $v<0<u$.

Now suppose $v \ge 0$. As $|x|^q$ is increasing for $x \ge 0$, \[ a|u|^q+(1-a)|v|^q \ge a\left(\frac{(q-1)^{q-1}}{a(1-a)}\right)^q=\frac{(q-1)^{q(q-1)}}{a^{q-1}(1-a)^q}. \] Using the work above, \[ \frac{(q-1)^{q(q-1)}}{a^{q-1}(1-a)^q} \ge \frac{(q-1)^{q(q-1)}}{\max\left\{a^{q-1}(1-a)^q,a^q(1-a)^{q-1}\right\}}>2^{q-1}>1, \] so the inequality holds whenever $v \ge 0$. The case $u \le 0$ is identical, which completes the proof.
\end{proof}

With this, we have the following key result.

\begin{lem}\label{pqboth}
Fix $q \in (1,2)$, a nonempty set $S=\{(u_1,v_1),\ldots,(u_k,v_k)\}$ of points in $[0,1] \times \mathbb R$, and $(x,y) \in [0,1] \times \mathbb R$ such that $u_1<\ldots<u_k$, $x \neq u_i$ for any $1 \le i \le k$, and $J_q\left[f_{S \cup \{(x,y)\}}\right] \le 1$. Let $p=2+\frac{1}{q-1}$. Then \[ J_q\left[f_{S \cup \{(x,y)\}}\right]-J_q[f_S] \ge |y-f_S(x)|^p. \]
\end{lem}

\begin{proof}
We first show that $|y-f_S(x)|>(q-1)^{q-1}$ and $x \in (u_1,u_k)$ cannot both hold. Suppose otherwise, so that there exists an integer $1 \le i<k$ such that $u_i<x<u_{i+1}$. We will derive a contradiction by showing $J_q\left[f_{S \cup \{(x,y)\}}\right]>1$. Clearly \[ J_q\left[f_{S \cup \{(x,y)\}}\right] \ge (x-u_i)\left|\frac{y-v_i}{x-u_i}\right|^q+(u_{i+1}-x)\left|\frac{v_{i+1}-y}{u_{i+1}-x}\right|^q. \]
Substituting $a=x-u_i,b=u_{i+1}-x,d=y-f_S(x)$, and $m=\frac{v_{i+1}-v_i}{u_{i+1}-u_i}$ as in Lemma \ref{2qboth}, this rewrites as \[ J_q\left[f_{S \cup \{(x,y)\}}\right] \ge a\left|m+\frac{d}{a}\right|^q+b\left|m-\frac{d}{b}\right|^q. \] As $a+b=u_{i+1}-u_i \le 1$ and $q>1$, \[ J_q\left[f_{S \cup \{(x,y)\}}\right] \ge \frac{a}{a+b}\left|(a+b)m+\frac{d(a+b)}{a}\right|^q+\frac{b}{a+b}\left|(a+b)m-\frac{d(a+b)}{b}\right|^q, \] and because \[ |d|>(q-1)^{q-1} \implies \left|d(a+b)\left(\frac{1}{a}+\frac{1}{b}\right)\right|=\frac{|d|(a+b)^2}{ab} \ge \frac{(q-1)^{q-1}}{\frac{a}{a+b} \cdot \frac{b}{a+b}}, \] applying Lemma \ref{uv} yields $J_q\left[f_{S \cup \{(x,y)\}}\right]>1$, contradiction.

Thus at least one of $|y-f_S(x)| \le (q-1)^{q-1}$ and $x \not \in (u_1,u_k)$ holds. If \[ |y-f_S(x)| \le (q-1)^{q-1} \implies (q-1)(y-f_S(x))^2 \ge |y-f_S(x)|^p, \] the result follows from Lemma \ref{2qboth}. Otherwise, assume without loss of generality that $x<u_1$ (the case $x>u_k$ is similar); then \[ J_q\left[f_{S \cup \{(x,y)\}}\right]-J_q[f_S]=(u_1-x)\left|\frac{v_1-y}{u_1-x}\right|^q \ge |v_1-y|^q \ge |v_1-y|^p=|y-f_S(x)|^p \] by Lemma \ref{linint1} (as $q<2<p$) and the result holds in this case as well.
\end{proof}

This immediately yields the following.

\begin{thm}\label{pq}
For any reals $q>1$ and $p \ge 2+\frac{1}{q-1}$, we have $\opt_p(\mathcal F_q)=1$.
\end{thm}

\begin{proof}
By Proposition \ref{pq1} and Corollary \ref{linintp}, it suffices to prove that for $q \in (1,2)$ and $p=2+\frac{1}{q-1}$, $\mathscr L_p(\text{LININT},\mathcal F_q) \le 1$. Fix $p=2+\frac{1}{q-1}$, a target function $f \in \mathcal F_q$, an integer $m \ge 1$, and a sequence of inputs $\sigma=(x_0,\ldots,x_m) \in [0,1]^{m+1}$. Assume without loss of generality that all $x_i$ are distinct. For $0 \le i \le m$, define $S_i=\{(x_0,f(x_0)),\ldots,(x_i,f(x_i))\}$, and suppose LININT produces guesses $\hat y_0,\ldots,\hat y_m \in \mathbb R$. By Lemma \ref{fsmin} and Lemma \ref{pqboth}, \[ 1 \ge J_q[f] \ge J_q\left[f_{S_m}\right]=\sum_{i=1}^m\left(J_q\left[f_{S_i}\right]-J_q\left[f_{S_{i-1}}\right]\right) \ge \sum_{i=1}^m|\hat y_i-f(x_i)|^p, \] so $\mathscr L_p(\text{LININT},f,\sigma) \le 1$ for any $f \in \mathcal F_q$ and $\sigma$. Thus $\mathscr L_p(\text{LININT},\mathcal F_q) \le 1$.
\end{proof}

\subsection{Sharp bounds for $p \in (1, 2)$}\label{pql2}

The paper \cite{kl} showed that $\opt_{1+\varepsilon}(\mathcal{F}_q) = O(\varepsilon^{-1})$ for all $\varepsilon \in (0, 1)$ and $q \ge 2$. In this section, we first improve their upper bound by proving that $\opt_{1+\varepsilon}(\mathcal{F}_q) = O(\varepsilon^{-\frac{1}{2}})$ for all $\varepsilon \in (0, 1)$ and $q \ge 2$. Then we show that this bound is sharp by proving that $\opt_{1+\varepsilon}(\mathcal{F}_q) = \Omega(\varepsilon^{-\frac{1}{2}})$ for all $q \ge 1$. In order to prove the upper bound, we use two lemmas from \cite{kl}. To state the lemmas and prove our upper bound, we use the following notation. Let $x_0, \ldots, x_m$ be any sequence of distinct elements of $[0, 1]$, and let $f \in \mathcal{F}_2$. Let $\hat{y}_1, \ldots, \hat{y}_m$ be LININT's predictions on trials $1, \ldots, m$. For each $i > 1$, let $d_i = \min_{j < i} |x_j - x_i|$ and let $e_i = |\hat{y}_i - f(x_i)|$.

\begin{lem}[\cite{kl}] \label{ei2di} 
For all positive integers $m$, we have $\sum_{i = 1}^m \frac{e_i^2}{d_i} \le 1$.
\end{lem}

\begin{lem}[\cite{kl}] \label{sumditop}
For all positive integers $m$ and real numbers $x > 1$, we have $\sum_{i = 1}^m d_i^x \le 1 + \frac{1}{2^x - 2}$.
\end{lem}

By combining Lemmas \ref{ei2di} and \ref{sumditop} with H\"older's inequality, we obtain the following sharp upper bound.

\begin{thm}\label{upperbound}
If $p = 1+\varepsilon \in (1,2)$, then $\opt_{p}(\mathcal{F}_2) = O(\varepsilon^{-\frac{1}{2}})$.
\end{thm}

\begin{proof}
First, note that \[\sum_{i = 1}^m e_i^p = \sum_{i = 1}^m \frac{e_i^p}{d_i^{\frac{p}{2}}} \cdot d_i^{\frac{p}{2}}.\] By H\"older's inequality, we have \begin{align*}
 \sum_{i = 1}^m \frac{e_i^p}{d_i^{\frac{p}{2}}} \cdot d_i^{\frac{p}{2}} \le
\left(\sum_{i = 1}^m \frac{e_i^2}{d_i}\right)^{\frac{p}{2}} \left( \sum_{i = 1}^m  d_i^{\frac{p}{2-p}}\right)^{1 - \frac{p}{2}}.
\end{align*} Note that $\sum_{i = 1}^m \frac{e_i^2}{d_i} \le 1$ by Lemma \ref{ei2di} and \[\sum_{i = 1}^m  d_i^{\frac{p}{2-p}}  \le 1 + \frac{1}{2^{\frac{p}{2-p}}-2}\] by Lemma \ref{sumditop}, since $p > 1$ implies that $\frac{p}{2-p} > 1$. Thus \[\left(\sum_{i = 1}^m  d_i^{\frac{p}{2-p}}\right)^{1 - \frac{p}{2}} \le \left(1 + \frac{1}{2^{\frac{p}{2-p}}-2}\right)^{1-\frac{p}{2}}.\] Let $\delta = \frac{p}{2-p}-1$, and note that $\frac{1}{\delta} = \frac{2-p}{2p - 2}$. Thus \[\left(1 + \frac{1}{2^{\frac{p}{2-p}}-2}\right)^{1-\frac{p}{2}} = \left(1 + \frac{1}{2^{1+\delta}-2}\right)^{1-\frac{p}{2}} = O\left(\left(1+\frac{1}{\delta}\right)^{\frac{2-p}{2}}\right)=O\left(\left(\frac{p}{2p-2}\right)^{\frac{2-p}{2}}\right),\] where the upper bound follows from the fact that $e^{\delta \ln{2}} \ge 1+\delta \ln{2}$. Thus we have proved that \[\sum_{i = 1}^m e_i^p  = O\left(\left(\frac{p}{2p-2}\right)^{\frac{2-p}{2}}\right),\] so $\opt_{p}(\mathcal{F}_2) = O\left(\left(2p-2\right)^{-\frac{2-p}{2}}\right)$, where we use the fact that $p^{\frac{2-p}{2}} = \Theta(1)$ for $p \in (1,2)$ to obtain the last bound. Since $p = 1+\varepsilon$, we have \[\opt_{p}(\mathcal{F}_2) = O\left(\left(2p-2\right)^{-\frac{2-p}{2}}\right) = O\left(\varepsilon^{-\frac{1-\varepsilon}{2}}\right) = O(\varepsilon^{-\frac{1}{2}}),\] where we use the fact that $\varepsilon^{\varepsilon} = \Theta(1)$ for $\varepsilon \in (0, 1)$ to obtain the last bound.
\end{proof}

We obtain the next corollary since $\opt_p(\mathcal{F}_{\infty}) \le \opt_p(\mathcal{F}_r) \le \opt_p(\mathcal{F}_q)$ whenever $1 \le q \le r$.

\begin{cor}
\label{mainupper}
If $\varepsilon \in (0, 1)$, then $\opt_{1+\varepsilon}(\mathcal{F}_{\infty}) = O(\varepsilon^{-\frac{1}{2}})$ and $\opt_{1+\varepsilon}(\mathcal{F}_q) = O(\varepsilon^{-\frac{1}{2}})$ for all $q \ge 2$, where the constant does not depend on $q$.
\end{cor}

In order to show that the last corollary is sharp up to a constant factor, we construct a family of functions in $\mathcal{F}_{\infty}$. Our proof uses the following lemma from \cite{kl} which was also used in \cite{long}.

\begin{lem}[\cite{kl}] \label{kl_jlem} 
Let $S \subseteq [0, 1] \times \mathbb{R}$ with $S = \left\{(u_i, v_i): 1 \le i \le m\right\}$ and $u_1 < u_2 < \cdots < u_m$. If $(x,y) \in [0, 1] \times \mathbb{R}$ and there exists $1 \le j \le m$ such that $|x-u_j| = |x-u_{j+1}| = \min_i |x-u_i|$, then $J_2[f_{S \cup \left\{(x,y) \right\}}] = J_2[f_S] + \frac{2(y-f_S(x))^2}{\min_i |x-u_i|}$.
\end{lem}

The method in the following proof is similar to one used in \cite{long} to obtain bounds for a finite variant of $\opt_1(\mathcal{F}_q)$ for $q \ge 2$ that depends on the number of trials $m$. 

\begin{thm}\label{lowerbound}
If $\varepsilon \in (0, 1)$, then $\opt_{1+\varepsilon}(\mathcal{F}_{\infty}) = \Omega(\varepsilon^{-\frac{1}{2}})$.
\end{thm}

\begin{proof}
Since $\opt_{1+\varepsilon}(\mathcal{F}_{\infty}) \ge 1$ for all $\varepsilon \in (0, 1)$, it suffices to prove the theorem for $\varepsilon \in \left(0, \frac{1}{2}\right)$. Define $x_0 = 1$ and $y_0 = 0$. For natural numbers $i, j$ with $0 \le j < 2^{i-1}$, define $x_{2^{i-1}+j} = \frac{1}{2^i}+\frac{j}{2^{i-1}}$. For each $i = 1, 2, \ldots$, we consider the trials for $x_{2^{i-1}}, \ldots, x_{2^i - 1}$ to be part of stage $i$, so that $x_1 = \frac{1}{2}$ is in stage $1$, $x_2 = \frac{1}{4}$ and $x_3 = \frac{3}{4}$ are in stage $2$, and so on.

Let $A$ be any algorithm for learning $\mathcal{F}_{\infty}$. Using $A$, we construct an infinite sequence of piecewise linear functions $f_0, f_1, \ldots \in \mathcal{F}_{\infty}$ and an infinite sequence of numbers $y_0, y_1, \ldots \in \mathbb{R}$ for which $f_t$ is consistent with the $x_k$ and $y_k$ values for $k \le t$ and $A$ has total $(1+\varepsilon)$-error at least \[\sum_{k = 1}^i 2^{k-2} \left(\frac{\sqrt{\varepsilon}(1-\varepsilon)^{\frac{k}{2}}}{2^{k+1}}\right)^{1+\varepsilon}\] after $i$ stages. This implies that \[\opt_{1+\varepsilon}(\mathcal{F}_{\infty}) \ge \sum_{k = 1}^{\infty} 2^{k-2} \left(\frac{\sqrt{\varepsilon}(1-\varepsilon)^{\frac{k}{2}}}{2^{k+1}}\right)^{1+\varepsilon}.\]

In order to analyze the functions $f_i$, we will also define and analyze another infinite sequence of piecewise linear functions $g_{i, j}$ with $0 \le j \le 2^{i-1}$ and another infinite sequence of numbers $v_1, v_2, \ldots \in \mathbb{R}$. We start by letting $f_0$ be the $0$-function. Next, we inductively define both sequences of piecewise linear functions.

Fix a stage $i$, and let $g_{i, 0} = f_{2^{i-1}-1}$. Let $t$ be a trial in stage $i$, and let $v_t$ be whichever of $f_{t-1}(x_t) \pm \frac{\sqrt{\varepsilon}(1-\varepsilon)^{\frac{i}{2}}}{2^{i+1}}$ is furthest from $\hat{y}_t$. Let $g_{i, t-2^{i-1}+1}$ be the function which linearly interpolates $\left\{(0, 0), (1, 0) \right\} \cup \left\{(x_s, y_s): s < 2^{i-1} \right\} \cup \left\{(x_s, v_s): 2^{i-1} \le s \le t \right\}$. 

For any $t \ge 1$, let $L_t$ and $R_t$ be the elements of $\left\{0, 1 \right\} \cup \left\{x_s: s < t \right\}$ that are closest to $x_t$ on the left and right respectively. If both $|v_t - f_{t-1}(L_t)| \le 2^{-i}$ and $|v_t - f_{t-1}(R_t)| \le 2^{-i}$, then let $y_t = v_t$. Otherwise we let $y_t = f_{t-1}(x_t)$. Finally, we define $f_t$ to be the function which linearly interpolates $\left\{(0, 0), (1, 0)\right\} \cup \left\{(x_s, y_s): s \le t\right\}$.

By definition, we have $f_t \in \mathcal{F}_{\infty}$ for each $t \ge 0$. We will prove next that for all $i, j$ we have $J_2[g_{i, j}] \le \frac{1}{4}$, and then we will use this to prove that $y_t = f_{t-1}(x_t)$ for at most half of the trials $t$ in stage $i$. The proof will use double induction, first on $i$ and then on $j$, and we will prove the stronger statement that \begin{equation}\label{stronger_statement}
J_2[g_{i, j}] \le \frac{\varepsilon}{4} \sum_{k = 0}^{i-1} (1-\varepsilon)^k + \frac{j \varepsilon (1-\varepsilon)^i }{2^{i+1}}.\end{equation}
In order to prove this statement, we will also prove that \begin{equation}\label{sum0toi-1of_f}J_2[f_{2^{i-1}-1}] \le \frac{\varepsilon}{4} \sum_{k = 0}^{i-1} (1-\varepsilon)^k\end{equation} for all $i \ge 1$. Note that this is equivalent to proving that \[J_2[g_{i, 0}] \le \frac{\varepsilon}{4} \sum_{k = 0}^{i-1} (1-\varepsilon)^k\] for all $i \ge 1$. Clearly this is true for $i = 1$, which is the base case of the induction on $i$. Fix some stage $i \ge 1$. We will assume that Inequality \ref{sum0toi-1of_f} is true for this fixed $i$, and use this to prove that \begin{equation}\label{sum0toiof_f}J_2[f_{2^{i}-1}] \le \frac{\varepsilon}{4} \sum_{k = 0}^{i} (1-\varepsilon)^k.\end{equation} In order to prove Inequality \ref{sum0toiof_f}, we will prove Inequality \ref{stronger_statement} for all $0 \le j \le 2^{i-1}$. This follows from the inductive hypothesis for $i$ and the definition of $g_{i, 0}$ when $j = 0$, which is the base case of the induction on $j$. Fix some integer $j$ with $0 \le j \le 2^{i-1}-1$ and assume that Inequality \ref{stronger_statement} is true for this fixed $j$. By Lemma \ref{kl_jlem}, we have \[J_2[g_{i, j+1}] = J_2[g_{i, j}]+ \frac{2\left(\frac{\sqrt{\varepsilon}(1-\varepsilon)^{\frac{i}{2}}}{2^{i+1}}\right)^2}{2^{-i}} = J_2[g_{i, j}]+\frac{\varepsilon (1-\varepsilon)^i}{2^{i+1}}.\] By the inductive hypothesis for $j$, we obtain \[J_2[g_{i, j+1}] \le \frac{\varepsilon}{4} \sum_{k = 0}^{i-1} (1-\varepsilon)^k+ \frac{j \varepsilon (1-\varepsilon)^i }{2^{i+1}} + \frac{\varepsilon (1-\varepsilon)^i}{2^{i+1}},\] which completes the inductive step for $j$. Substituting $j = 2^{i-1}$, we obtain \[J_2[g_{i,2^{i-1}}] \le  \frac{\varepsilon}{4} \sum_{k = 0}^{i} (1-\varepsilon)^k.\] Note that Lemma \ref{kl_jlem} implies that \[J_2[f_{2^{i-1}-1+j}] \le J_2[g_{i, j}]\] for all $j = 0, \ldots, 2^{i-1}$, so we obtain Inequality \ref{sum0toiof_f}, which completes the inductive step for $i$. By Inequality \ref{stronger_statement}, we obtain \[J_2[g_{i, j}] \le \frac{\varepsilon}{4} \sum_{k = 0}^{i-1} (1-\varepsilon)^k + \frac{\varepsilon (1-\varepsilon)^i }{4} = \frac{\varepsilon}{4} \sum_{k = 0}^{i} (1-\varepsilon)^k\] for all $j$ with $0 \le j \le 2^{i-1}$. Note that \[\frac{\varepsilon}{4} \sum_{k = 0}^{i} (1-\varepsilon)^k <  \frac{\varepsilon}{4} \sum_{k = 0}^{\infty} (1-\varepsilon)^k = \frac{1}{4}.\] Now that we have shown that $J_2[g_{i, j}] \le \frac{1}{4}$, we are ready to prove for each $i \ge 1$ that $y_t = f_{t-1}(x_t)$ for at most half of the trials $t$ in stage $i$. For each trial $t$ with $y_t = f_{t-1}(x_t)$, note that the absolute value of the slope of $g_{i, t-2^{i-1}+1}$ must exceed $1$ in at least one of the intervals of length $2^{-i}$ on either side of $x_t$. If $y_t = f_{t-1}(x_t)$ for at least $b$ of the trials in stage $i$, then restricting to intervals of slope at least $1$ implies that $J_2[g_{i, 2^{i-1}}] \ge b 2^{-i}$. Since $J_2[g_{i, 2^{i-1}}] \le \frac{1}{4}$, we must have $b \le 2^{i-2}$. Thus during stage $i$, there are at most $2^{i-2}$ trials $t$ with $y_t = f_{t-1}(x_t)$, which implies that there are at least $2^{i-2}$ trials with $y_t = v_t$. In each of those trials, $A$ was off by at least $\frac{\sqrt{\varepsilon}(1-\varepsilon)^{\frac{i}{2}}}{2^{i+1}}$, so the total $(1+\varepsilon)$-error of $A$ after $i$ stages is at least $\sum_{k = 1}^i 2^{k-2} \left(\frac{\sqrt{\varepsilon}(1-\varepsilon)^{\frac{k}{2}}}{2^{k+1}}\right)^{1+\varepsilon}$. Thus

\[
\opt_{1+\varepsilon}(\mathcal{F}_{\infty}) \ge \sum_{k = 1}^{\infty} 2^{k-2} \left(\frac{\sqrt{\varepsilon}(1-\varepsilon)^{\frac{k}{2}}}{2^{k+1}}\right)^{1+\varepsilon} = 
\frac{\frac{1}{2}\left(\frac{\sqrt{\varepsilon(1-\varepsilon)}}{4}\right)^{1+\varepsilon}}{1-2\left(\frac{\sqrt{1-\varepsilon}}{2}\right)^{1+\varepsilon}} =
\Omega\left(\frac{(\varepsilon(1-\varepsilon))^{\frac{1+\varepsilon}{2}}}{1-2\left(\frac{\sqrt{1-\varepsilon}}{2}\right)^{1+\varepsilon}}\right).
\] Since $\varepsilon^{\varepsilon} = \Theta(1)$ and $(1-\varepsilon)^{1+\varepsilon} = \Theta(1)$ for $\varepsilon \in \left(0, \frac{1}{2}\right)$, we have $\opt_{1+\varepsilon}(\mathcal{F}_{\infty}) = \Omega\left(\frac{\sqrt{\varepsilon}}{1-2\left(\frac{\sqrt{1-\varepsilon}}{2}\right)^{1+\varepsilon}}\right)$. Since $2^{\varepsilon} = \Theta(1)$ for $\varepsilon \in \left(0, \frac{1}{2}\right)$, we have $\opt_{1+\varepsilon}(\mathcal{F}_{\infty}) = \Omega\left(\frac{\sqrt{\varepsilon}}{2^{\varepsilon}-\sqrt{1-\varepsilon}^{1+\varepsilon}}\right)$. Note that $(1-\varepsilon)^{\frac{1+\varepsilon}{2}} \ge 1-\varepsilon(1+\varepsilon)$ for $\varepsilon \in \left(0, \frac{1}{2}\right)$. To check this, note that it is true when $\varepsilon = 0$, and the derivative of $(1-\varepsilon)^{\frac{1+\varepsilon}{2}} - (1-\varepsilon(1+\varepsilon))$ is \[2 \varepsilon +1+ (1-\varepsilon)^{\frac{1+\varepsilon}{2}}\left(\frac{1}{2}\ln(1-\varepsilon)-\frac{1}{2} -\frac{\varepsilon}{1-\varepsilon}\right) > 0\] for $\varepsilon \in \left(0, \frac{1}{2}\right)$. Thus, $\opt_{1+\varepsilon}(\mathcal{F}_{\infty}) = \Omega\left(\frac{\sqrt{\varepsilon}}{2^{\varepsilon}-1+\varepsilon(1+\varepsilon)}\right)$. Also note that $2^{\varepsilon} \le 1+\varepsilon$ for $\varepsilon \in (0, 1)$. Equality holds at $\varepsilon = 0$ and $\varepsilon = 1$, and the derivative of $1+\varepsilon - 2^{\varepsilon}$ is $1-2^{\varepsilon}\ln{2}$, which is positive for $\varepsilon \in \left(0, \frac{-\ln{\ln{2}}}{\ln{2}}\right)$ and negative for $\varepsilon \in \left(\frac{-\ln{\ln{2}}}{\ln{2}}, 1\right)$. Thus $2^{\varepsilon}-1+\varepsilon(1+\varepsilon) < 3\varepsilon$, so $\opt_{1+\varepsilon}(\mathcal{F}_{\infty}) = \Omega(\varepsilon^{-\frac{1}{2}})$.
\end{proof}

The next corollary follows from Theorem \ref{lowerbound}, again using the fact that $\opt_p(\mathcal{F}_{\infty}) \le \opt_p(\mathcal{F}_r) \le \opt_p(\mathcal{F}_q)$ whenever $1 \le q \le r$.

\begin{cor}
\label{mainlower}
If $\varepsilon \in (0, 1)$, then $\opt_{1+\varepsilon}(\mathcal{F}_q) = \Omega(\varepsilon^{-\frac{1}{2}})$ for all $q \ge 1$, where the constant does not depend on $q$.
\end{cor}

Combining Corollaries \ref{mainupper} and \ref{mainlower}, we have the following theorem.

\begin{thm}
If $\varepsilon \in (0, 1)$, then $\opt_{1+\varepsilon}(\mathcal{F}_{\infty}) = \Theta(\varepsilon^{-\frac{1}{2}})$ and $\opt_{1+\varepsilon}(\mathcal{F}_q) = \Theta(\varepsilon^{-\frac{1}{2}})$ for all $q \ge 2$, where the constant in the bound does not depend on $q$.
\end{thm}

\section{A multi-variable generalization}\label{3}

In this section, we prove several results on $\opt_p(\mathcal F_{q,d})$. First, we prove a simple lower bound for $\opt_p(\mathcal F_{q,d})$ in terms of $\opt_p(\mathcal F_q)$. 

\begin{prop}
\label{dlower}
For any positive integer $d$, real number $p > 0$, and $q \in [1, \infty) \cup \left\{\infty\right\}$, we have \[ \opt_p(\mathcal F_{q,d}) \ge d^p \cdot \opt_p(\mathcal F_q). \]
\end{prop}

\begin{proof}
If $\opt_p(\mathcal F_{q,d})=\infty$, there is nothing to prove, and if $\opt_p(\mathcal F_q)=\infty$, it is clear, by restricting the inputs $\mathbf x_i$ to the set $\{c\mathbf e_1: c \in [0,1]\} \subset [0,1]^d$ (where $\mathbf e_1 \in [0,1]^d$ has a $1$ in the first component and a $0$ in the rest), that $\opt_p(\mathcal F_{q,d})=\infty$ as well.

Now suppose that both $\opt_p(\mathcal F_{q,d})$ and $\opt_p(\mathcal F_q)$ are finite. Fix any algorithm $A$ for learning $\mathcal F_{q,d}$. Let $\mathbf{1}$ be the all-ones $d$-tuple and let $a(x_i : (x_0, z_0),\dots,(x_{i-1},z_{i-1}))$ denote the output of $A$ given the input $x_i \mathbf{1}$ after learning the pairs $(x_j \mathbf{1}, z_j)$ for $j < i$, given that there is a function in $\mathcal{F}_{q,d}$ which passes through the points $(x_j \mathbf{1}, z_j)$ for $j < i$. Then let $A’$ be the algorithm for learning $\mathcal{F}_q$ which, given the input $x_i$ after learning the pairs $(x_j, w_j)$ for $j < i$, returns the output \[ a'(x_i : (x_0, w_0),\dots,(x_{i-1},w_{i-1}))=\frac{a(x_i : (x_0, d w_0),\dots,(x_{i-1},d w_{i-1}))}{d}, \] given that there is a function in $\mathcal{F}_q$ which passes through the points $(x_j, w_j)$ for $j < i$.

Fix $\varepsilon > 0$. Then there exist $f \in \mathcal{F}_q$ and a sequence of inputs $x_0, x_1, \ldots, x_M$ such that \[ \sum_{i = 1}^{M} |a'(x_i : (x_0, f(x_0)),\dots,(x_{i-1},f(x_{i-1})))-f(x_i)|^p \ge \opt_p(\mathcal{F}_q)-\varepsilon.\] Against $A$, the adversary uses the function \[\gamma(a_1,\ldots,a_d) = \sum_{i = 1}^d f(a_i)\] with the inputs $x_0 \mathbf{1}, x_1 \mathbf{1}, \ldots, x_M \mathbf{1}$. First, suppose that $q \in [1, \infty)$. Observe that for any $1 \le k \le d$ and $d-1$ reals $x_i \in [0,1]$, where $1 \le i \le d$ but $i \neq k$, \[ \int_0^1 \left| \frac{\text{d}\gamma}{\text{d}x_k} \right|^q\text{d}x_k = \int_0^1|f'(x)|^q\text{d}x \le 1 \] since $f \in \mathcal F_q$; hence, $\gamma \in \mathcal F_{q,d}$. Next, suppose that $q = \infty$. Observe that $\left| \frac{\text{d}\gamma}{\text{d}x_k} \right| = |f'(x_k)| \le 1$ for all $x_k \in [0,1]$ since $f \in \mathcal F_q$; hence, in this case we also have $\gamma \in \mathcal F_{q,d}$. To finish the proof, let \[e_i = |a'(x_i: (x_0, f(x_0)),\dots,(x_{i-1},f(x_{i-1})))-f(x_i)|\] and \[k_i = |a(x_i : (x_0, \gamma(x_0 \mathbf{1})),\dots,(x_{i-1},\gamma(x_{i-1} \mathbf{1})))-\gamma(x_i \mathbf{1})|\] for each $i$. Thus, \[k_i = |d a'(x_i : (x_0, f(x_0)),\dots,(x_{i-1},f(x_{i-1})))-d f(x_i)| = d e_i\] and \[ \sum_{i = 1}^{M} e_i^p \ge \opt_p(\mathcal{F}_q)-\varepsilon.\] Hence, \[ \sum_{i = 1}^{M} k_i^p \ge d^p(\opt_p(\mathcal{F}_q)-\varepsilon).\] Taking $\varepsilon \to 0$ finishes the proof.
\end{proof}

The next corollary follows from Proposition \ref{dlower} since $\opt_p(\mathcal F_{1}) = \infty$ \cite{kl}. 

\begin{cor}
For any positive integer $d$ and real number $p > 0$, we have $\opt_p(\mathcal F_{1,d}) = \infty$. 
\end{cor}

Now we directly prove some results about $\opt_p(\mathcal F_{\infty,d})$, depending on whether $p<d$ or $p>d$. The main negative result is the following.

\begin{thm}\label{p_inf_d}
Let $d>0$ be an integer and $p$ be a real number with $0 < p<d$. Then $\opt_p(\mathcal F_{\infty,d})=\infty$.
\end{thm}

\begin{proof}
Fix any algorithm $A$ for learning $\mathcal F_{\infty,d}$. Then choose any integer $n \ge 1$, and let $S$ be the set of reals $0<r<1$ such that $2n r$ is an odd integer (so $|S|=n$). The adversary first reveals $f(0,\ldots,0)=0$, then chooses $\mathbf x_i$ ranging over all elements of $S^d$ in lexicographic order, receives input $\hat y_i$ from $A$, and reveals $f(\mathbf x_i)=\pm\frac{1}{2n}$, whichever is farther from $\hat y_i$.

Let $\{x\}=x-\lfloor x \rfloor$ denote the fractional part of $x$. At the end of the $n^{d}+1$ trials, the algorithm's revealed values of $f$ are consistent with a function $f: [0,1]^d \to \mathbb R$ given by \[ f(x_1,\ldots,x_d)=\pm \frac{1}{n} \min_{1 \le i \le d} \left(\min \left(\{n x_i\},\{-n x_i\}\right) \right), \] where the signs $\pm$ are chosen such that $f(\mathbf x)$ agrees with the adversary's outputs for any $\mathbf x=(x_1,\ldots,x_d) \in S^d$ and $f$ has constant sign in any region \[ \left(\frac{n_1}{n},\frac{n_1+1}{n}\right) \times \ldots \times \left(\frac{n_d}{n},\frac{n_d+1}{n}\right) \subset [0,1]^d \] for integers $0 \le n_i < n$. The consistency follows since $\{n x_i\} = \{-n x_i\} = \frac{1}{2}$ for all $x_i \in S$.

First, we show $f \in \mathcal F_{\infty,d}$. Fix any $1 \le i \le d$ and $\mathbf x=(x_1,\ldots,x_{d-1}) \in [0,1]^{d-1}$, and consider the function $g:[0,1] \to \mathbb R$ given by $g(x)=f(\mathbf x')$, where $\mathbf x' \in [0,1]^d$ is formed by inserting $x$ into the $i^{\text{th}}$ position of $\mathbf x$. Then $g$ is given by \[ g(x)=\pm \frac{1}{n}\min\left(\min\left(\{n x\},\{-n x\}\right),M\right), \] where \[ M=\min_{1 \le i \le d-1} \left(\min \left(\{n x_i\},\{-n x_i\}\right) \right). \] Evidently $g$ is piecewise linear, with finitely many points where $g'$ is not defined and $|g'(x)| = 1$ or $g'(x) = 0$ everywhere else by definition of $g$; moreover, since the function $\min(\{n x\},\{-n x\})$ is continuous, it follows that $g$ is continuous. Hence $g \in \mathcal F_\infty$. Since this holds for any choice of $1 \le i \le d$ and $\mathbf x \in [0,1]^d$, it follows that $f \in \mathcal F_{\infty,d}$.

Now we find a lower bound for the error the adversary can guarantee. There are $n^{d}$ trials past the first, each of which has $|\hat y_i-f(\mathbf x_i)| \ge \frac{1}{2n}$; hence the adversary guarantees \[ \sum_{i>0}|\hat y_i-f(\mathbf x_i)|^p \ge \frac{n^{d}}{(2n)^{p}}=\frac{1}{2^p} \cdot n^{d-p}. \] Because $p<d$, this grows arbitrarily large as $n$ increases; hence $\opt_p(\mathcal F_{\infty,d})=\infty$.
\end{proof}

As $\mathcal F_\infty \subseteq \mathcal F_q$ implies that $\mathcal F_{\infty,d} \subseteq \mathcal F_{q,d}$ for any $q \ge 1$, this bound extends to $q \neq \infty$.

\begin{cor}
Let $d>0$ be an integer and $p$ be a real number with $0 < p<d$. For any $q \ge 1$, we have $\opt_p(\mathcal F_{q,d})=\infty$.
\end{cor}

In order to establish upper bounds on $\opt_p(\mathcal F_{\infty,d})$, we prove the following lemma.

\begin{lem}\label{tri}
For $f \in \mathcal F_{\infty,d}$ and $\mathbf x_1=(x_{1,1},\ldots,x_{d,1}),\mathbf x_2=(x_{1,2},\ldots,x_{d,2}) \in [0,1]^d$, we have \[ |f(\mathbf x_1)-f(\mathbf x_2)| \le \sum_{i=1}^d|x_{i,1}-x_{i,2}|. \]
\end{lem}

\begin{proof}
Fix such $f,\mathbf x_1,\mathbf x_2$. Define a sequence of $\mathbf x_i' \in [0,1]^d$, for $0 \le i \le d$, such that $\mathbf x_i'$ has its first $i$ components equal to the first $i$ components of $\mathbf x_2$ and its last $d-i$ components equal to the last $d-i$ components of $\mathbf x_1$ (so $\mathbf x_0'=\mathbf x_1$ and $\mathbf x_d'=\mathbf x_2$). By the triangle inequality, \[ |f(\mathbf x_1)-f(\mathbf x_2)|=\left|\sum_{i=1}^d \left(f(\mathbf x_{i-1}')-f(\mathbf x_i')\right)\right| \le \sum_{i=1}^d \left|f(\mathbf x_{i-1}')-f(\mathbf x_i')\right|. \] Now consider any $1 \le i \le d$. Note that $\mathbf x_{i-1}$ and $\mathbf x_i'$ only differ in their $i^{\text{th}}$ components, with one being $x_{i,1}$ and the other being $x_{i,2}$. Then by definition of $\mathcal F_{\infty,d}$ and using the fact that for $g \in \mathcal F_\infty$ and $x_1,x_2 \in [0,1]$, $|g(x_1)-g(x_2)| \le |x_1-x_2|$, it follows that $\left|f(\mathbf x_{i-1}')-f(\mathbf x_i')\right| \le |x_{i,1}-x_{i,2}|$. Summing over $1 \le i \le d$ yields the result.
\end{proof}

Lemma \ref{tri} makes the class $\mathcal F_{\infty,d}$ particularly nice to work with. Using a nearest neighbor algorithm, we establish the following upper bound.

\begin{thm}
\label{dpupper}
Suppose $p>d$. Then $\opt_p(\mathcal F_{\infty,d}) \le \frac{(2^d-1)d^p}{1-\frac{2^d}{2^p}}$.
\end{thm}

\begin{proof}
Consider the algorithm $A$ which guesses $0$ on the first input and, on trial $i$ (after receiving inputs $\mathbf x_0,\ldots,\mathbf x_{i-1}$), picks the least index $0 \le j<i$ which minimizes the $L_1$ distance between $\mathbf x_j$ and $\mathbf x_i$ and guesses $\hat y_i=f(\mathbf x_j)$. We will show $\mathscr L_p(A,\mathcal F_{\infty,d}) \le \frac{(2^d-1)d^p}{1-\frac{2^d}{2^p}}$.

Fix $f \in \mathcal F_{\infty,d}$ and a sequence $\mathbf x_0,\ldots,\mathbf x_m$ of $\mathbf x_i \in [0,1]^d$. Assume all the $\mathbf x_i$ are distinct. Then for each $1 \le i \le m$, there exists a least integer $k_i$ such that, if $[0,1]^d$ is divided into $2^{k_id}$ regions given by \[ \{ (x_1,\ldots,x_d) \in [0,1]^d: n_i \le 2^{k_i}x_i \le n_i+1 \} \] over all $d$-tuples $(n_1,\ldots,n_d)$ of integers $0 \le n_i<2^{k_i}$, then $\mathbf x_i$ is not in the same region as any of $\mathbf x_0,\ldots,\mathbf x_{i-1}$. Note that because $\mathbf x_0$ and $\mathbf x_i$ are both in $[0,1]^d$ for any $1 \le i \le m$, all $k_i$ are at least $1$. For each integer $k \ge 1$, let $c_k$ be the number of integers $1 \le i \le m$ such that $k_i=k$. By the Pigeonhole Principle, for any fixed integer $k \ge 0$, there exist at most $2^{kd}-1$ indices $1 \le i \le m$ such that $k_i \le k$; otherwise, at least $2^{kd}+1$ of the $\mathbf x_i$ (including $\mathbf x_0$) would be the first within their containing length-$2^{-k}$ hypercube region. Thus \begin{equation} \label{cksum2} \sum_{k=1}^{K}c_k \le 2^{Kd}-1 \end{equation} for any integer $K \ge 1$. Moreover, for any $1 \le i \le m$, $\mathbf x_i$ lies in the same length-$2^{-(k_i-1)}$ hypercube as one of $\mathbf x_0,\ldots,\mathbf x_{i-1}$, and this hypercube has $L_1$ distance $\frac{d}{2^{k_i-1}}$ between two of its opposite vertices, so by Lemma \ref{tri}, \begin{equation}\label{yifxiabs}|\hat y_i-f(\mathbf x_i)| \le \frac{d}{2^{{k_i}-1}}.\end{equation} Combining these, \begin{align*}
    \sum_{i=1}^m |\hat y_i-f(\mathbf x_i)|^p & \le \sum_{k \ge 1} c_k\left(\frac{d}{2^{k-1}}\right)^p
    = \sum_{K \ge 1} \left[ \left(\sum_{k=1}^K c_k\right) \left(\left(\frac{d}{2^{K-1}}\right)^p-\left(\frac{d}{2^K}\right)^p\right) \right]
    \\ &\le d^p\sum_{K \ge 1}(2^{Kd}-1)(2^{-p(K-1)}-2^{-pK})
    = d^p(2^p-1)\sum_{K \ge 1}2^{-pK}(2^{Kd}-1)
    \\ &= d^p(2^p-1)\left(\frac{2^{d-p}}{1-2^{d-p}}-\frac{2^{-p}}{1-2^{-p}}\right)
    = \frac{(2^d-1)d^p}{1-\frac{2^d}{2^p}}.
\end{align*} This holds for all $f \in \mathcal F_{\infty,d}$ and sequences of $\mathbf x_i$; hence $\mathscr L_p(A,\mathcal F_{\infty,d}) \le \frac{(2^d-1)d^p}{1-\frac{2^d}{2^p}}$.
\end{proof}

The next corollary follows from Proposition \ref{dlower} and Theorem \ref{dpupper} since $\opt_p(\mathcal F_{\infty}) = 1$ for all $p \ge 2$. 

\begin{cor}
For any fixed positive integer $d$ and real number $p \ge d+1$, we have $\opt_p(\mathcal F_{\infty,d}) = \Theta(d^p)$, where the constant in the upper bound depends only on $d$.
\end{cor}

We can also use Theorems \ref{p_inf_d} and \ref{dpupper} to obtain sharp bounds on the worst-case errors for learning $\mathcal{F}_{\infty,d}$ when the number of trials is bounded.

\begin{cor}
Let $d>0$ be an integer and $p$ be a real number with $0 < p<d$. Then $\opt_p(\mathcal F_{\infty,d}, m)=\Theta(m^{1-\frac{p}{d}})$, where the constants in the bounds depend on $p$ and $d$.
\end{cor}

\begin{proof}
By Theorem \ref{p_inf_d}, we have $\opt_p(\mathcal F_{\infty,d}, m) \ge \frac{1}{2^p} n^{d-p}$ for $n = \lfloor m^{1/d} \rfloor$, so we obtain the lower bound $\opt_p(\mathcal F_{\infty,d}, m) \ge \frac{1}{2^p} m^{\frac{d-p}{d}}(1-o(1))$. For the upper bound, we use the algorithm and notation of Theorem \ref{dpupper} with $K = \left\lceil \frac{ \log_2(m+1)}{d}\right\rceil$ to obtain \begin{align*}
    \sum_{i=1}^m |\hat y_i-f(\mathbf x_i)|^p & \le \sum_{i=1}^m  \left(\frac{d}{2^{k_i-1}}\right)^p 
    \le \sum_{k = 1}^{K} (2^{kd}-2^{(k-1)d})\left(\frac{d}{2^{k-1}}\right)^p
    \\ &= d^p (2^d - 1) \sum_{k = 1}^{K} 2^{(k-1)(d-p)}
    = d^p (2^d - 1) \frac{2^{K(d-p)}-1}{2^{d-p}-1}
    \\ &< \frac{d^p (2^d - 1)2^{d-p}}{2^{d-p}-1}m^{\frac{d-p}{d}}(1+o(1)),
\end{align*} where the first inequality follows from Inequality \ref{yifxiabs} and the second inequality follows from Inequality \ref{cksum2} since $\left(\frac{d}{2^{k-1}}\right)^p$ is decreasing in $k$. Thus \[ \opt_p(\mathcal F_{\infty,d}, m) \le \frac{d^p (2^d - 1)2^{d-p}}{2^{d-p}-1}m^{\frac{d-p}{d}}(1+o(1)). \]
\end{proof}

\section{Discussion and open problems}\label{s:open}

With the results in this paper, the value of $\opt_p(\mathcal{F}_q)$ is now bounded up to a constant factor for all $p, q \ge 1$ except when $q \in (1, 2)$ and $p \in (1, 2) \cup (2, 2+\frac{1}{q-1})$. In particular, by combining the results in this paper with the results in \cite{kl}, we now know that $\opt_p(\mathcal{F}_q) = 1$ for all $(p , q)$ that lie in the following regions. 

\begin{itemize}
    \item $p, q \ge 2$
    \item $q \in (1, 2)$ and $p \ge 2+\frac{1}{q-1}$
\end{itemize}

In addition to investigating the regions in which $\opt_p(\mathcal{F}_q)$ is not bounded up to a constant factor, it remains to narrow the constant gap between the upper and lower bounds for $\opt_{1+\varepsilon}(\mathcal{F}_q) = \Theta(\varepsilon^{-\frac{1}{2}})$ when $\varepsilon \in (0, 1)$ and $q \in [2, \infty) \cup \left\{ \infty \right\}$. Another similar problem is to narrow the constant gap between the upper and lower bounds for $\opt_2(\mathcal{F}_q) = \Theta(\varepsilon^{-1})$ when $q \in (1, 2)$.

The results in this paper also help characterize the values of $(p, q)$ for which $\opt_p(\mathcal{F}_q)$ is finite. Before this paper, it was only known that $\opt_2(\mathcal{F}_q)$ is finite for $p > 1$ and $q \ge 2$, and $\opt_p(\mathcal{F}_q) = \infty$ when $p = 1$ or $q = 1$. With our new results, we now know that $\opt_p(\mathcal{F}_q)$ is also finite when $p \ge 2$ and $q > 1$. We make the following conjecture about this problem.

\begin{conj}
For all $p > 1$ and $q > 1$, $\opt_p(\mathcal{F}_q)$ is finite.
\end{conj}

Besides the new results about smooth functions of a single variable, we also introduced a generalization of the model to multi-variable functions and found some bounds for this multi-variable online learning scenario. We showed that $\opt_p(\mathcal F_{\infty,d})$ is infinite when $0 < p < d$ and finite when $p > d$, but it remains to determine whether $\opt_d(\mathcal F_{\infty,d})$ is finite for $d > 1$. For finite $q \ge 1$ and $0 < p < d$, we also know that $\opt_p(\mathcal F_{q,d})$ is infinite, but it remains to determine whether $\opt_p(\mathcal F_{q,d})$ is finite for $p \ge d$ and $q \in (1, \infty)$. In addition, we proved for any fixed positive integer $d$ that $\opt_p(\mathcal F_{\infty,d}) = \Theta(d^p)$ for $p \ge d+1$, where the constant in the upper bound depends on $d$. The multiplicative gap between the upper and lower bounds is $2^{d+1} - 2$. We conjecture that the lower bound is sharp for $p$ sufficiently large with respect to $d$ and $q$.

\begin{conj}
    For all $q \in [1, \infty) \cup \left\{\infty\right\}$, for all positive integers $d$, and for all real numbers $p$ sufficiently large with respect to $q$ and $d$, we have $\opt_p(\mathcal F_{q,d}) = d^p$.
\end{conj}

The papers \cite{kl} and \cite{long} investigated $\opt_1(\mathcal{F}_q, m)$ for $q \ge 2$, where $m$ is the number of trials. It would be natural to study $\opt_p(\mathcal{F}_q, m)$ for $p = 1+\varepsilon$ with $0 < \varepsilon < 1$ and $q \ge 1$, since $\opt_{1+\varepsilon}(\mathcal{F}_q)$ can grow arbitrarily large as $\varepsilon \rightarrow 0$. We bounded $\opt_p(\mathcal F_{\infty,d}, m)$ up to a constant factor for any fixed positive integer $d$ and fixed real number $p$ with $0 < p<d$, but the constants in the bounds depend on $p$ and $d$. It remains to narrow the gap between the upper bound of $\frac{d^p (2^d - 1)2^{d-p}}{2^{d-p}-1}m^{\frac{d-p}{d}}(1+o(1))$ and the lower bound of $\frac{1}{2^p} m^{\frac{d-p}{d}}(1-o(1))$. It would also be interesting to investigate $\opt_p(\mathcal F_{q,d}, m)$ for finite values of $q$.

Another possible direction would be to investigate families of smooth functions with additional restrictions. For example, let $\mathcal{E}_{q} \subseteq \mathcal{F}_q$ be the family of exponential functions $f(x) = e^{a x +b}$ with $f \in \mathcal{F}_q$. 

\begin{prop}
For all $p > 0$ and $q \ge 1$, we have $\opt_p(\mathcal{E}_q) = 1$.
\end{prop}

\begin{proof}
    The upper bound $\opt_p(\mathcal{E}_q) \le 1$ follows by Lemma \ref{linint1}, since the learner knows the function after two rounds with different inputs and the first round does not count for the total error. For the lower bound, consider an adversary that chooses some $\varepsilon \in (0,1)$, defines $\varrho = 1-\sqrt[1-\varepsilon]{1-\varepsilon}$, and reveals $f(0) = -\frac{1-\varepsilon}{\ln(1-\varepsilon)}$. On the second turn, they either reveal $f(1) = -\frac{1}{\ln(1-\varepsilon)}$ or $f(1) = -\frac{1-\varrho}{\ln(1-\varrho)}$, whichever maximizes the error for the learner's guess.

    If $f(1) = -\frac{1}{\ln(1-\varepsilon)}$, then $f(x) = e^{a x +b}$ with $a = -\ln(1-\varepsilon)$ and $b = \ln(1-\varepsilon)-\ln(-\ln(1-\varepsilon))$. Note that $f'(x) = a e^{a x + b} \in [1-\varepsilon, 1]$ for all $x \in [0,1]$, so $f \in \mathcal{E}_q$ for all $q \ge 1$. If $f(1) =  -\frac{1-\varrho}{\ln(1-\varrho)}$, then $f(x) = e^{a x +b}$ with $a = \ln(1-\varrho)$ and $b = -\ln(-\ln(1-\varrho))$. Note that $f'(x) = a e^{a x + b} \in [-1, -1+\varrho]$ for all $x \in [0,1]$, so $f \in \mathcal{E}_q$ for all $q \ge 1$. Moreover, note that \[\lim_{\varepsilon \rightarrow 0} \left(-\frac{1}{\ln(1-\varepsilon)}+\frac{1-\varrho}{\ln(1-\varrho)}\right) = \lim_{\varepsilon \rightarrow 0} \left( \frac{-1+(1-\varepsilon)\sqrt[1-\varepsilon]{1-\varepsilon}}{\ln(1-\varepsilon)} \right) = 2,\] by L'H\^{o}pital's rule. Thus $\opt_p(\mathcal{E}_q) \ge 1$.
\end{proof}

Let $\mathcal{P}_{q, m} \subseteq \mathcal{F}_q$ be the family of polynomial functions $f \in \mathcal{F}_q$ of degree at most $m$. It is easy to see that $\opt_p(\mathcal{P}_{q, 1}) = 1$ for all $p > 0$ and $q \ge 1$, but it would be interesting to investigate $\opt_p(\mathcal{P}_{q, m})$ for $m > 1$. Note that we have $\opt_p(\mathcal{P}_{q, m}) \le \opt_p(\mathcal{F}_q, m)$ for all $p > 0$, $q \ge 1$, and $m \ge 1$, since the learner will know $f \in \mathcal{P}_{q, m}$ with certainty after being tested on $m+1$ different inputs. Let $\mathcal{P}_{q} \subseteq \mathcal{F}_q$ be the family of all polynomial functions $f \in \mathcal{F}_q$. We make the following conjecture about this family.

\begin{conj}
For all $p > 0$ and $q \ge 1$, we have $\opt_p(\mathcal{P}_q) = \opt_p(\mathcal{F}_q)$.
\end{conj}

Some other possible subsets of $\mathcal{F}_q$ that could be investigated are piecewise functions with at most $k$ pieces where the pieces are polynomials of degree at most $m$, sums of exponential functions, and sums of trigonometric functions.

Finally, we return to the problem from the introduction of predicting the next day's temperature range at a given location. In particular, consider the single-variable problem where we predict the next day's temperature range based only on the time of year. An issue with using the model from \cite{kl} for this temperature prediction problem is that the same input for time of year could have different outputs for the temperature range in different years. A more realistic way to model this problem would be to choose the output from a probability distribution which depends on the input. In order for the learner to be able to guarantee a finite bound on the worst-case error, the number of trials would be bounded and restrictions would be placed on the probability distribution. For example, the density function for the probability distribution could be required to have smoothness properties like the functions from \cite{kl}, and the support of the density function could be required to be a subset of $[0,r]$ for some $r > 0$. Investigating such a model would be an interesting direction for future research. Note that this model reduces to the model from \cite{kl} when the support consists of a single point. 

\section{Acknowledgments}

Most of this research was performed in PRIMES 2022. We thank the organizers for this research opportunity. Our paper subsumes \cite{geneson}, which proved Theorem \ref{mainth}. We also thank the anonymous reviewers for helpful comments which improved the clarity and presentation of the results in this paper.

\end{document}